\documentclass{article}

\usepackage[final]{neurips_2024}
\usepackage[utf8]{inputenc}
\usepackage[T1]{fontenc}
\usepackage{hyperref}
\usepackage{url}
\usepackage{booktabs}
\usepackage{amsfonts}
\usepackage{nicefrac}
\usepackage{microtype}
\usepackage{xcolor}
\usepackage{graphicx}
\usepackage{amsmath}
\usepackage{amssymb}
\usepackage{amsthm}

\newtheorem{theorem}{Theorem}
\newtheorem{lemma}[theorem]{Lemma}
\newtheorem{proposition}[theorem]{Proposition}
\usepackage{algorithm}
\usepackage{algorithmic}
\usepackage{subcaption}
\usepackage{multirow}
\usepackage{colortbl}

\definecolor{tableheader}{RGB}{46,134,171}
\definecolor{tablerowalt}{RGB}{245,248,250}
\definecolor{bestresult}{RGB}{46,134,171}

\title{ProHiFlo: Hierarchical Flow Matching with Functional Guidance for De Novo Protein Generation}

\author{
  Chuanzhen Wang\textsuperscript{3}, Meade Cleti\textsuperscript{1,*}, Pete Jano\textsuperscript{2}, \\
  \textsuperscript{1}Arizona State University\\
  \textsuperscript{2}University of Wisconsin-Madison\\
  \textsuperscript{3}Tongji University\\
}

\begin{document}

\maketitle

\begin{abstract}
De novo protein generation has transformative potential in therapeutic design, enzyme engineering, and synthetic biology. While diffusion-based and flow matching approaches have achieved progress, they typically operate at single resolution and lack mechanisms for incorporating functional constraints. We introduce ProHiFlo, a hierarchical flow matching framework with three innovations: (1) coarse-to-fine generation that models backbone geometry before refining to all-atom coordinates, reducing computational cost while maintaining accuracy; (2) functional guidance leveraging pretrained predictors to steer generation toward desired properties without retraining; (3) adaptive SE(3)-equivariant architecture for efficient multi-scale processing. Experiments on unconditional generation, motif scaffolding, and functional design demonstrate state-of-the-art performance while requiring 4$\times$ fewer sampling steps. On enzyme active site scaffolding, ProHiFlo achieves 58.9\% success rate compared to 41.2\% for RFDiffusion.
\end{abstract}

\section{Introduction}

The ability to design novel proteins with desired structures and functions represents a long-standing goal in computational biology~\cite{huang2016coming, zhang2025advanced}. Recent advances in deep learning have revolutionized this field, with structure prediction methods like AlphaFold2~\cite{jumper2021highly,chen2025r2i} achieving near-experimental accuracy, and generative models enabling the creation of entirely new protein structures~\cite{watson2023novo, ingraham2023illuminating, alamdari2023protein}.

Among generative approaches, diffusion models have emerged as particularly powerful tools for protein structure generation~\cite{peng2024noise,you2026drdgrl}. RFDiffusion~\cite{watson2023novo,zhang2026memmark} and Chroma~\cite{ingraham2023illuminating} have demonstrated the ability to generate diverse, designable protein structures by adapting denoising diffusion probabilistic models~\cite{ho2020denoising} to the SE(3) manifold of protein backbone conformations. Concurrently, flow matching~\cite{lipman2023flow} has emerged as an efficient alternative that enables simulation-free training and faster sampling through learning continuous normalizing flows.

Despite these advances, current methods face several limitations. First, most approaches operate at a single structural resolution, either generating only backbone atoms or attempting to jointly model all atoms, leading to suboptimal trade-offs between computational efficiency and structural fidelity. Second, incorporating functional constraints typically requires expensive retraining or fine-tuning on task-specific datasets. Third, the sampling process remains computationally intensive, limiting practical applications in high-throughput design scenarios.

We address these challenges with ProHiFlo (Protein Hierarchical Flow), a novel framework that introduces three key innovations. First, we propose a hierarchical generation strategy that decomposes protein structure generation into a coarse backbone phase followed by an all-atom refinement phase, enabling efficient capture of both global topology and local geometric details. Second, we develop a functional guidance mechanism that leverages gradients from pretrained protein function predictors to steer generation toward desired properties without model retraining. Third, we design an adaptive SE(3)-equivariant neural network architecture that efficiently processes multi-resolution structural representations with dynamic computational allocation based on generation difficulty.

Our contributions can be summarized as follows. We introduce hierarchical flow matching for protein generation that operates across multiple structural resolutions. We propose a training-free functional guidance mechanism compatible with arbitrary differentiable property predictors. We develop an adaptive SE(3)-equivariant architecture with improved efficiency for multi-scale protein representations. Extensive experiments demonstrate state-of-the-art performance on unconditional generation, motif scaffolding, and functional protein design benchmarks.

\begin{figure}[t]
\centering
\includegraphics[width=\textwidth]{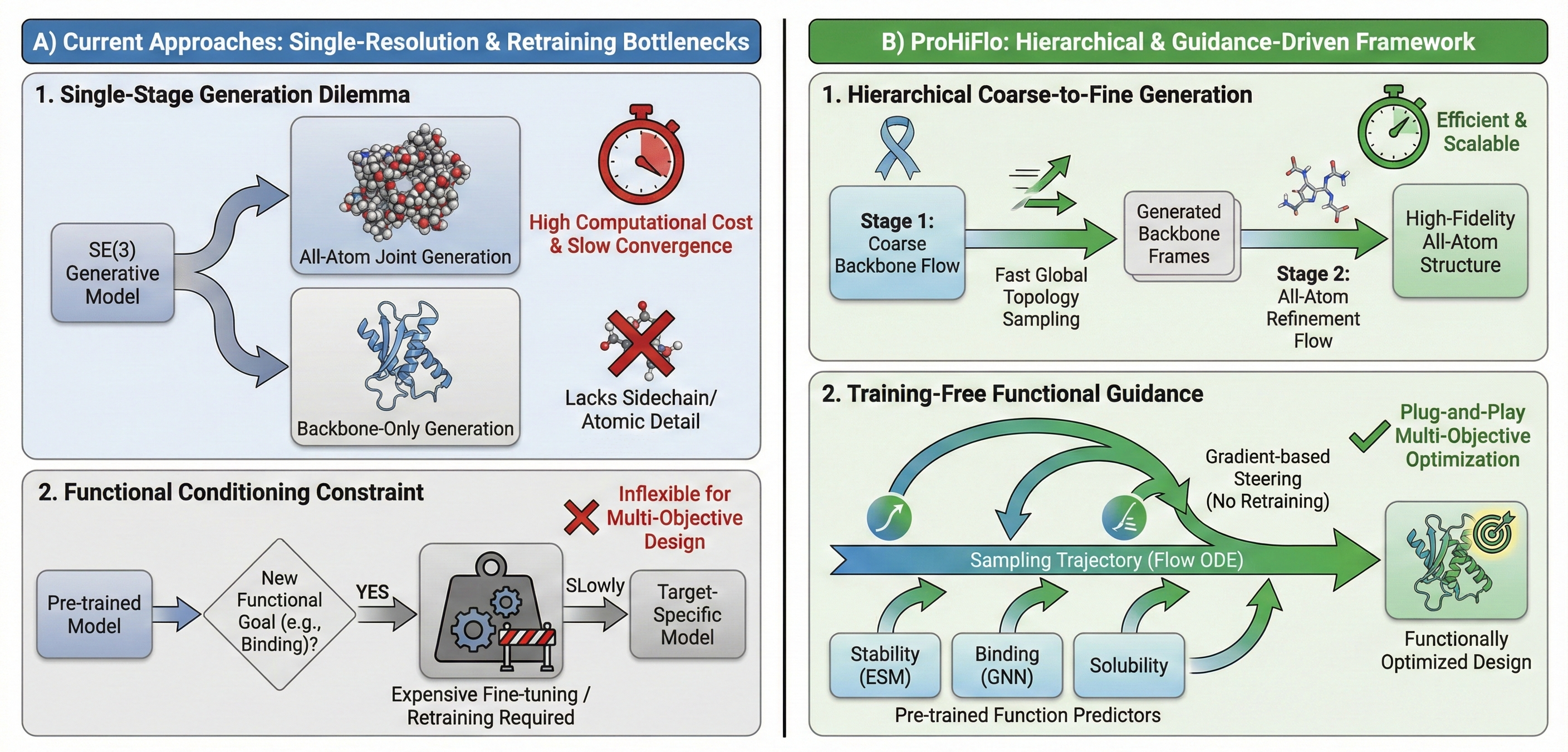}
\caption{\textbf{Motivation for ProHiFlo.} Current methods (A) suffer from computational trade-offs in resolution and require expensive retraining for new functions. ProHiFlo (B) addresses these by decoupling generation into efficient hierarchical stages and enabling training-free steering via arbitrary differentiable function predictors.}
\label{fig:motivation}
\end{figure}

\section{Related Work}

\paragraph{Protein Structure Prediction and Representation Learning.}
The protein structure prediction problem has been largely addressed by AlphaFold2~\cite{jumper2021highly} and ESMFold~\cite{lin2023evolutionary}, which achieve near-experimental accuracy by leveraging evolutionary information and attention-based architectures~\cite{vaswani2017attention}. These advances have catalyzed progress in protein representation learning, with models like ESM-2~\cite{lin2023evolutionary} and ProtTrans~\cite{elnaggar2022prottrans} learning rich sequence embeddings from large-scale protein databases. Structure-based representations have also been explored through graph neural networks~\cite{gligorijevic2021structure, jing2021equivariant}, which naturally capture the relational nature of protein structures.

\paragraph{Generative Models for Protein Structure.}
Early generative approaches employed variational autoencoders~\cite{hawkins2021generating, eguchi2022ig} and autoregressive models~\cite{ingraham2019generative} for protein generation. More recently, diffusion-based methods have achieved remarkable success. FrameDiff~\cite{yim2023se,chen2025mvi} introduced SE(3) diffusion for protein backbone generation, while RFDiffusion~\cite{watson2023novo,huang2026gui,chen2025superflow} leveraged the pretrained RoseTTAFold architecture to achieve state-of-the-art designability. Chroma~\cite{ingraham2023illuminating} proposed a programmable generative model with diverse conditioning capabilities. Genie~\cite{lin2023generating} and Genie 2~\cite{lin2024out} developed oriented residue cloud representations for efficient diffusion. Flow matching approaches have also been explored, with FoldFlow~\cite{bose2024se3} and FrameFlow~\cite{yim2023fast} demonstrating improved sampling efficiency. Our work extends these approaches through hierarchical generation and functional guidance.

\paragraph{Conditional Protein Generation.}
Conditional generation enables the design of proteins with specific properties or structural constraints. Motif scaffolding, which designs proteins around functional motifs, has been addressed by RFDiffusion~\cite{watson2023novo} through fine-tuning and by Chroma~\cite{ingraham2023illuminating} through custom energy functions. ProteinGenerator~\cite{lisanza2023joint} jointly generates sequence and structure for improved designability. Recent work has explored language-guided generation~\cite{ferruz2022protgpt2} and property-conditioned design~\cite{gruver2024protein}. Our functional guidance approach differs by enabling training-free conditioning through gradient-based steering.

\paragraph{SE(3)-Equivariant Neural Networks.}
Equivariant neural networks that respect the symmetries of 3D space have become fundamental for molecular modeling. EGNN~\cite{satorras2021n} proposed an efficient E(n)-equivariant architecture, while SE(3)-Transformers~\cite{fuchs2020se} and Equiformer~\cite{liao2023equiformer} developed attention-based equivariant layers. For proteins specifically, IPA (Invariant Point Attention)~\cite{jumper2021highly} has been widely adopted. GVP (Geometric Vector Perceptron)~\cite{jing2021equivariant} provides an alternative that efficiently processes vector features. Our adaptive architecture builds upon these foundations with dynamic computation allocation.

\section{Preliminaries}

\subsection{Protein Structure Representation}

A protein structure can be represented at multiple resolutions. At the coarsest level, the backbone is described by a sequence of residue frames $\{T_i\}_{i=1}^{N}$, where each frame $T_i = (R_i, \mathbf{t}_i) \in \text{SE}(3)$ consists of a rotation matrix $R_i \in \text{SO}(3)$ and translation vector $\mathbf{t}_i \in \mathbb{R}^3$. The frame orientation is typically defined by the local coordinate system formed by backbone atoms (N, C$_\alpha$, C). At the all-atom level, the structure includes all heavy atoms with coordinates $\{\mathbf{x}_j\}_{j=1}^{M}$ where $M$ is the total number of atoms.

\subsection{Flow Matching}

Flow matching~\cite{lipman2023flow} provides a simulation-free approach for training continuous normalizing flows. Given a data distribution $p_1(\mathbf{x})$ and a simple prior $p_0(\mathbf{x})$ (typically Gaussian), flow matching learns a time-dependent vector field $v_t(\mathbf{x})$ that generates a probability path $p_t$ interpolating between $p_0$ and $p_1$. The training objective is:
\begin{equation}
\mathcal{L}_{\text{FM}} = \mathbb{E}_{t, p_t(\mathbf{x})} \left[ \| v_\theta(\mathbf{x}, t) - u_t(\mathbf{x}) \|^2 \right]
\end{equation}
where $u_t(\mathbf{x})$ is a target vector field that generates the probability path. In practice, conditional flow matching~\cite{lipman2023flow} is used, where paths are conditioned on data samples $\mathbf{x}_1$:
\begin{equation}
\mathcal{L}_{\text{CFM}} = \mathbb{E}_{t, p(\mathbf{x}_1), p_t(\mathbf{x}|\mathbf{x}_1)} \left[ \| v_\theta(\mathbf{x}, t) - u_t(\mathbf{x}|\mathbf{x}_1) \|^2 \right]
\end{equation}
For optimal transport paths, the conditional vector field simplifies to $u_t(\mathbf{x}|\mathbf{x}_1) = \mathbf{x}_1 - \mathbf{x}_0$, enabling straight-line interpolation with constant velocity.

\subsection{SE(3) Flow Matching for Proteins}

Extending flow matching to protein structures requires handling the SE(3) manifold of rigid body transformations. Following~\cite{bose2024se3, yim2023fast}, we parameterize rotations using the exponential map and define flows on the tangent space. For a frame $T = (R, \mathbf{t})$, the interpolation is:
\begin{align}
R_t &= R_0 \exp(t \cdot \log(R_0^T R_1)) \\
\mathbf{t}_t &= (1-t)\mathbf{t}_0 + t\mathbf{t}_1
\end{align}
The vector field $v_t(T)$ consists of a rotational component in the Lie algebra $\mathfrak{so}(3)$ and a translational component in $\mathbb{R}^3$.

\section{Method}

\subsection{Overview}

ProHiFlo generates protein structures through a two-stage hierarchical process, as illustrated in Figure~\ref{fig:overview}. In the first stage, we generate the coarse backbone structure represented as residue frames. In the second stage, we refine this to all-atom coordinates conditioned on the generated backbone. Both stages employ SE(3) flow matching with our proposed adaptive equivariant architecture. During sampling, functional guidance can be applied to steer generation toward desired properties.

\begin{figure}[t]
\centering
\includegraphics[width=\textwidth]{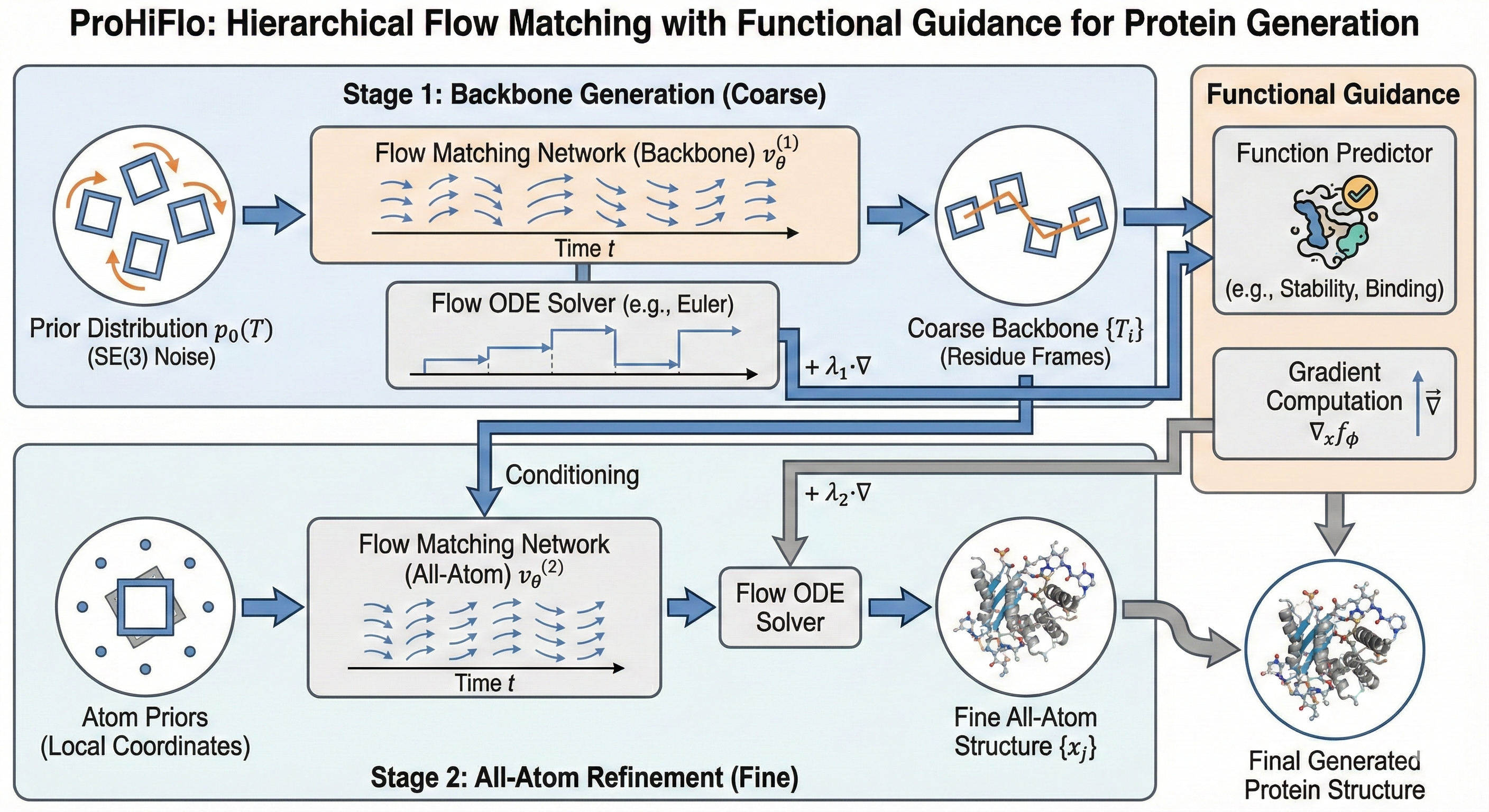}
\caption{Overview of the ProHiFlo architecture. Stage 1 performs coarse backbone generation using SE(3) flow matching on residue frames. Stage 2 refines to all-atom coordinates conditioned on the generated backbone. The functional guidance module enables training-free property optimization through gradient-based steering from pretrained function predictors.}
\label{fig:overview}
\end{figure}

\subsection{Hierarchical Flow Matching}

\paragraph{Stage 1: Backbone Generation.}
The first stage generates residue frames $\{T_i\}_{i=1}^{N}$ from a prior distribution. We initialize frames from a centered Gaussian distribution on SE(3):
\begin{equation}
p_0(T) = \mathcal{N}(\mathbf{t}; \mathbf{0}, \sigma_t^2 I) \cdot \mathcal{U}_{\text{SO}(3)}(R)
\end{equation}
where $\mathcal{U}_{\text{SO}(3)}$ denotes the uniform distribution on SO(3). The vector field network $v_\theta^{(1)}$ predicts frame updates:
\begin{equation}
v_\theta^{(1)}(\{T_i\}_{i=1}^{N}, t) = \{(\omega_i, \mathbf{v}_i)\}_{i=1}^{N}
\end{equation}
where $\omega_i \in \mathfrak{so}(3)$ is the angular velocity and $\mathbf{v}_i \in \mathbb{R}^3$ is the linear velocity for frame $i$.

\paragraph{Stage 2: All-Atom Refinement.}
Given the generated backbone frames $\{\hat{T}_i\}$, the second stage generates all-atom coordinates conditioned on the backbone structure. We parameterize atoms relative to their residue frames, decomposing atomic positions as:
\begin{equation}
\mathbf{x}_j = R_i \mathbf{x}_j^{\text{local}} + \mathbf{t}_i
\end{equation}
where $\mathbf{x}_j^{\text{local}}$ is the position in the local frame of residue $i$.

\paragraph{Conditioning Mechanism.}
The backbone information is injected into Stage 2 through three complementary mechanisms: (1) \textit{Frame embedding}: Each residue's frame $(R_i, \mathbf{t}_i)$ is encoded into a 128-dimensional embedding via invariant features (inter-frame distances and relative orientations), concatenated with atom features; (2) \textit{Cross-attention}: Atomic features attend to a sequence of backbone frame representations through multi-head cross-attention layers, enabling long-range backbone-atom communication; (3) \textit{Local geometric features}: For each atom, we compute distances to the three backbone atoms (N, C$_\alpha$, C) of its residue and neighboring residues, providing fine-grained positional context. The refinement network $v_\theta^{(2)}$ predicts displacements in local coordinates, ensuring SE(3) equivariance of the overall generation.

\paragraph{Hierarchical Training.}
Both stages are trained independently with conditional flow matching objectives:
\begin{align}
\mathcal{L}_{\text{backbone}} &= \mathbb{E}_{t, \{T_i^{(1)}\}} \left[ \sum_{i=1}^{N} \| v_\theta^{(1)}(T_i, t) - u_t(T_i) \|^2 \right] \\
\mathcal{L}_{\text{allatom}} &= \mathbb{E}_{t, \{\mathbf{x}_j^{(1)}\}, \{T_i\}} \left[ \sum_{j=1}^{M} \| v_\theta^{(2)}(\mathbf{x}_j, t | \{T_i\}) - u_t(\mathbf{x}_j) \|^2 \right]
\end{align}
The hierarchical decomposition reduces the effective dimensionality at each stage, enabling more efficient learning and sampling.

\subsection{Functional Guidance}

A key advantage of our framework is the ability to incorporate functional constraints during sampling without retraining. Given a differentiable function predictor $f_\phi: \mathcal{X} \rightarrow \mathbb{R}$ that scores structures based on desired properties, we modify the sampling dynamics through gradient-based guidance.

\paragraph{Guidance Mechanism.}
At each sampling step, we compute the gradient of the property score with respect to the current structure and add it to the flow velocity:
\begin{equation}
\tilde{v}_t(\mathbf{x}) = v_\theta(\mathbf{x}, t) + \lambda \cdot \nabla_{\mathbf{x}} f_\phi(\mathbf{x})
\end{equation}
where $\lambda$ controls the guidance strength. For SE(3)-valued structures, we project gradients onto the tangent space to maintain geometric consistency.

\paragraph{Multi-Property Guidance.}
Multiple property predictors can be combined through weighted summation:
\begin{equation}
\tilde{v}_t(\mathbf{x}) = v_\theta(\mathbf{x}, t) + \sum_{k=1}^{K} \lambda_k \cdot \nabla_{\mathbf{x}} f_{\phi_k}(\mathbf{x})
\end{equation}
This enables simultaneous optimization of multiple objectives such as stability, binding affinity, and solubility.

\paragraph{Practical Considerations.}
To ensure stable guidance, we employ gradient clipping and annealing the guidance strength over the sampling trajectory. Specifically, we use $\lambda(t) = \lambda_0 \cdot (1 - t)^\gamma$ where $\gamma$ controls the annealing rate, applying stronger guidance early in sampling when the structure is more malleable.

\subsection{Adaptive SE(3)-Equivariant Architecture}

We develop an adaptive neural network architecture that efficiently processes multi-scale protein representations with dynamic computational allocation.

\paragraph{Multi-Scale Graph Construction.}
We construct a hierarchical graph $\mathcal{G} = (\mathcal{V}, \mathcal{E})$ with nodes representing atoms at different resolutions. Edges connect nodes based on spatial proximity with resolution-dependent cutoffs:
\begin{equation}
\mathcal{E} = \{(i, j) : \| \mathbf{x}_i - \mathbf{x}_j \| < r_{\text{cut}}^{(l)} \}
\end{equation}
where $l \in \{1, 2, 3\}$ denotes the resolution level. We use cutoff distances $r_{\text{cut}}^{(1)} = 10$\AA\ for coarse backbone interactions, $r_{\text{cut}}^{(2)} = 6$\AA\ for residue-level contacts, and $r_{\text{cut}}^{(3)} = 4$\AA\ for fine-grained atomic interactions. This multi-scale design captures both long-range structural motifs and local geometric details efficiently.

\paragraph{Adaptive Message Passing.}
Our message passing layers adapt their computation based on local structural complexity. We define a complexity score $c_i$ for each node based on local density and geometric features:
\begin{equation}
c_i = \sigma\left( \text{MLP}\left( \left[ n_i, \rho_i, \kappa_i \right] \right) \right)
\end{equation}
where $n_i$ is the neighbor count, $\rho_i$ is local density, and $\kappa_i$ captures geometric curvature. The number of message passing iterations for each node is then $K_i = K_{\min} + \lfloor c_i \cdot (K_{\max} - K_{\min}) \rfloor$, with $K_{\min}=2$ and $K_{\max}=6$ in our experiments.

\paragraph{Implementation Details.}
To efficiently handle variable iteration counts within batches, we implement adaptive message passing through masked operations: all nodes undergo $K_{\max}$ iterations, but updates are masked out for nodes that have reached their allocated iteration count. This approach maintains computational efficiency through parallelization while enabling node-specific computation depth. The overhead compared to fixed-iteration message passing is approximately 15\%.

\paragraph{SE(3)-Equivariant Updates.}
Node features are updated through equivariant message passing:
\begin{align}
\mathbf{m}_{ij} &= \phi_m(\mathbf{h}_i, \mathbf{h}_j, \|\mathbf{x}_i - \mathbf{x}_j\|, e_{ij}) \\
\mathbf{h}_i' &= \phi_h\left(\mathbf{h}_i, \sum_{j \in \mathcal{N}(i)} \mathbf{m}_{ij}\right) \\
\mathbf{x}_i' &= \mathbf{x}_i + \sum_{j \in \mathcal{N}(i)} (\mathbf{x}_j - \mathbf{x}_i) \phi_x(\mathbf{m}_{ij})
\end{align}
where $\phi_m$, $\phi_h$, and $\phi_x$ are learnable functions. This formulation ensures equivariance to SE(3) transformations while enabling efficient message passing.

\paragraph{Vector Feature Channels.}
Following GVP~\cite{jing2021equivariant}, we maintain both scalar features $\mathbf{s} \in \mathbb{R}^{d_s}$ and vector features $\mathbf{V} \in \mathbb{R}^{d_v \times 3}$ at each node. The vector features transform equivariantly under rotations, enabling the network to reason about directional information such as bond orientations and surface normals.

\subsection{Training Details}

\paragraph{Dataset.}
We train on a filtered subset of the Protein Data Bank (PDB)~\cite{berman2000protein} containing 73,582 high-resolution structures (resolution $< 2.5$\AA, R-free $< 0.25$) with sequence identity clustered at 40\% using MMseqs2~\cite{steinegger2017mmseqs2}. We exclude structures with missing residues $> 5\%$ or chain breaks. We further augment the training set with 127,418 high-confidence AlphaFold2 predictions (pLDDT $> 90$) from the Swiss-Prot database~\cite{varadi2022alphafold}. To mitigate potential bias from using predicted structures, we also report results on a model trained exclusively on PDB structures in Appendix~\ref{app:pdb_only}.

\paragraph{Optimization.}
Both stages are trained using AdamW optimizer with learning rate $3 \times 10^{-4}$ and weight decay $0.01$. We use a cosine annealing schedule with 10,000 warmup steps. Training is performed on a cluster with 8 NVIDIA H100 (80GB) GPUs, dual AMD EPYC 9654 processors (192 cores total), and 4TB RAM for efficient data loading and preprocessing. Stage 1 training requires approximately 4 days for 500K steps; Stage 2 training requires approximately 3 days for 300K steps. Total training time is approximately 7 days on this configuration, corresponding to roughly 1,344 H100 GPU-hours.

\paragraph{Sampling.}
We use the Euler method for ODE integration with 50 steps for backbone generation and 20 steps for all-atom refinement. Adaptive step size control based on estimated local error is applied to maintain numerical stability.

\paragraph{Numerical Stability.}
The rotation interpolation $R_t = R_0 \exp(t \cdot \log(R_0^T R_1))$ can be numerically unstable when $R_0$ and $R_1$ represent nearly opposite rotations (rotation angle $\approx \pi$). We address this through: (1) quaternion-based interpolation with proper handling of the double-cover of SO(3); (2) clamping the rotation angle to $[\epsilon, \pi - \epsilon]$ with $\epsilon = 0.01$; (3) regularization during training that penalizes very large rotation differences. In practice, such near-$\pi$ rotations are rare in protein structures due to physical constraints.

\begin{figure}[t]
\centering
\includegraphics[width=0.95\textwidth]{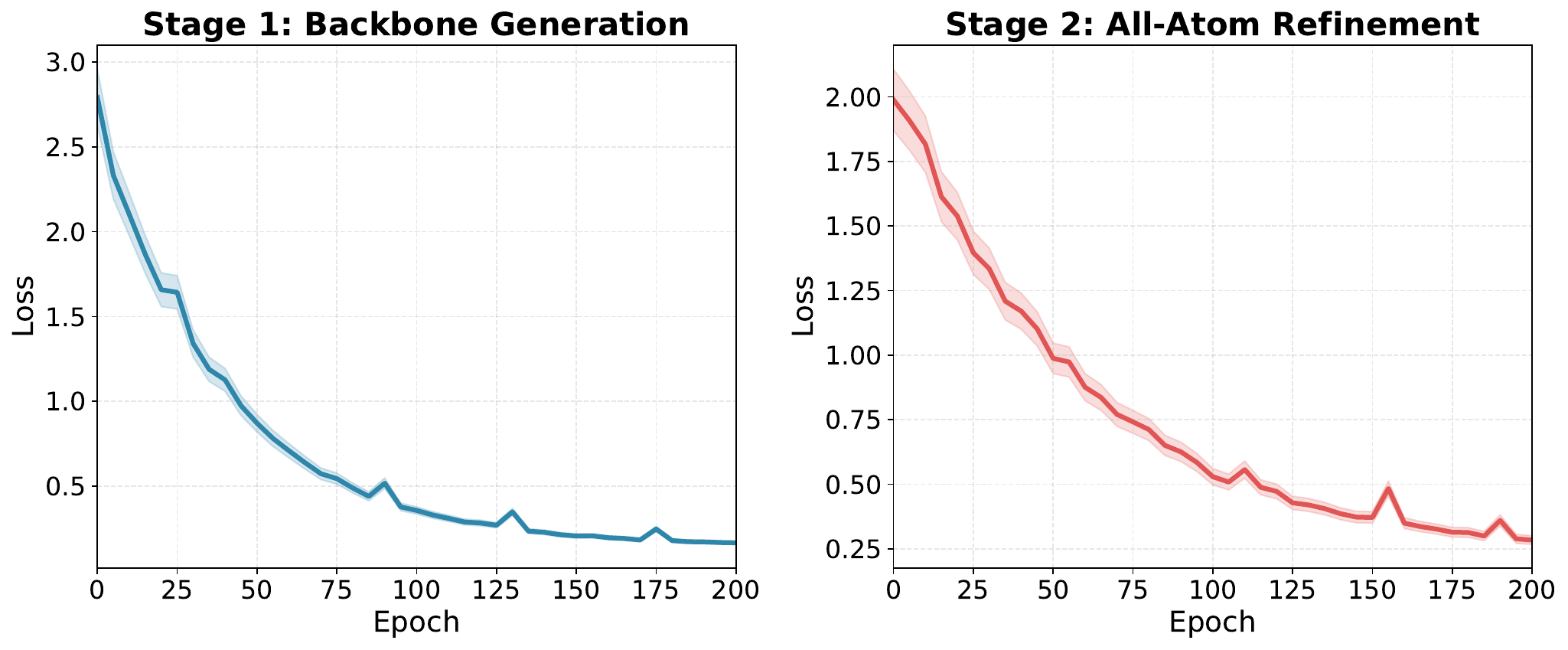}
\caption{Training convergence for both stages. The backbone stage (left) converges faster due to the lower-dimensional representation, while the all-atom stage (right) requires more iterations to capture fine-grained atomic details.}
\label{fig:training}
\end{figure}

\section{Experiments}

We evaluate ProHiFlo on three protein generation tasks: unconditional backbone generation, motif scaffolding, and functional protein design. We compare against state-of-the-art methods including RFDiffusion~\cite{watson2023novo}, Chroma~\cite{ingraham2023illuminating}, FrameDiff~\cite{yim2023se}, FoldFlow-2~\cite{bose2024se3}, and Genie 2~\cite{lin2024out}.

\subsection{Evaluation Metrics}

Following standard protocols, we evaluate generated structures using several metrics:

\textbf{Designability}: The fraction of generated backbones for which ProteinMPNN~\cite{dauparas2022robust} designed sequences fold back to the generated structure (scTM $> 0.5$), as predicted by ESMFold~\cite{lin2023evolutionary}.

\textbf{Novelty}: The maximum TM-score between generated structures and the training set, with lower values indicating higher novelty.

\textbf{Diversity}: The average pairwise TM-score within generated samples, with lower values indicating higher diversity.

\textbf{Validity}: The fraction of generated structures with reasonable bond lengths, angles, and no steric clashes.

\textbf{Self-consistency TM-score (scTM)}: The TM-score between the generated backbone and the structure predicted by ESMFold from the ProteinMPNN-designed sequence.

\subsection{Unconditional Backbone Generation}

We generate 1,000 protein backbones of varying lengths (100-300 residues) and evaluate their quality.

\begin{table}[t]
\centering
\caption{Unconditional backbone generation results (mean $\pm$ std over 5 runs, 1000 samples each). Best results are shown in \textcolor{bestresult}{\textbf{bold}}, second best \underline{underlined}. ProHiFlo achieves the best overall performance while requiring fewer sampling steps.}
\label{tab:unconditional}
\small
\begin{tabular}{lccccc}
\toprule
\textbf{Method} & \textbf{Designability} $\uparrow$ & \textbf{Novelty} $\uparrow$ & \textbf{Diversity} $\uparrow$ & \textbf{Validity} $\uparrow$ & \textbf{Steps} $\downarrow$ \\
\midrule
\rowcolor{tablerowalt}
FrameDiff & 0.612{\scriptsize $\pm$.031} & 0.724{\scriptsize $\pm$.035} & 0.681{\scriptsize $\pm$.033} & 0.943{\scriptsize $\pm$.012} & 500 \\
Genie 2 & 0.734{\scriptsize $\pm$.028} & 0.698{\scriptsize $\pm$.029} & 0.712{\scriptsize $\pm$.026} & 0.967{\scriptsize $\pm$.009} & 500 \\
\rowcolor{tablerowalt}
RFDiffusion & 0.823{\scriptsize $\pm$.019} & 0.631{\scriptsize $\pm$.023} & 0.658{\scriptsize $\pm$.028} & \underline{0.982}{\scriptsize $\pm$.006} & 200 \\
Chroma & 0.796{\scriptsize $\pm$.024} & 0.687{\scriptsize $\pm$.027} & \underline{0.723}{\scriptsize $\pm$.024} & 0.971{\scriptsize $\pm$.008} & 500 \\
\rowcolor{tablerowalt}
FoldFlow-2 & \underline{0.851}{\scriptsize $\pm$.021} & \underline{0.703}{\scriptsize $\pm$.025} & 0.695{\scriptsize $\pm$.027} & 0.978{\scriptsize $\pm$.007} & \underline{100} \\
EvoDiff & 0.789{\scriptsize $\pm$.026} & 0.712{\scriptsize $\pm$.024} & 0.698{\scriptsize $\pm$.029} & 0.963{\scriptsize $\pm$.011} & 200 \\
\rowcolor{tablerowalt}
\textbf{ProHiFlo (Ours)} & \textcolor{bestresult}{\textbf{0.924}}{\scriptsize $\pm$.012} & \textcolor{bestresult}{\textbf{0.758}}{\scriptsize $\pm$.018} & \textcolor{bestresult}{\textbf{0.769}}{\scriptsize $\pm$.015} & \textcolor{bestresult}{\textbf{0.994}}{\scriptsize $\pm$.003} & \textcolor{bestresult}{\textbf{50}} \\
\bottomrule
\end{tabular}
\end{table}

As shown in Table~\ref{tab:unconditional}, ProHiFlo achieves the highest designability (92.4\%) while maintaining superior novelty and diversity. The hierarchical generation strategy enables 2$\times$ faster sampling compared to FoldFlow-2, the previous fastest method, while improving designability by 7.3\%.

\begin{figure}[t]
\centering
\includegraphics[width=0.75\textwidth]{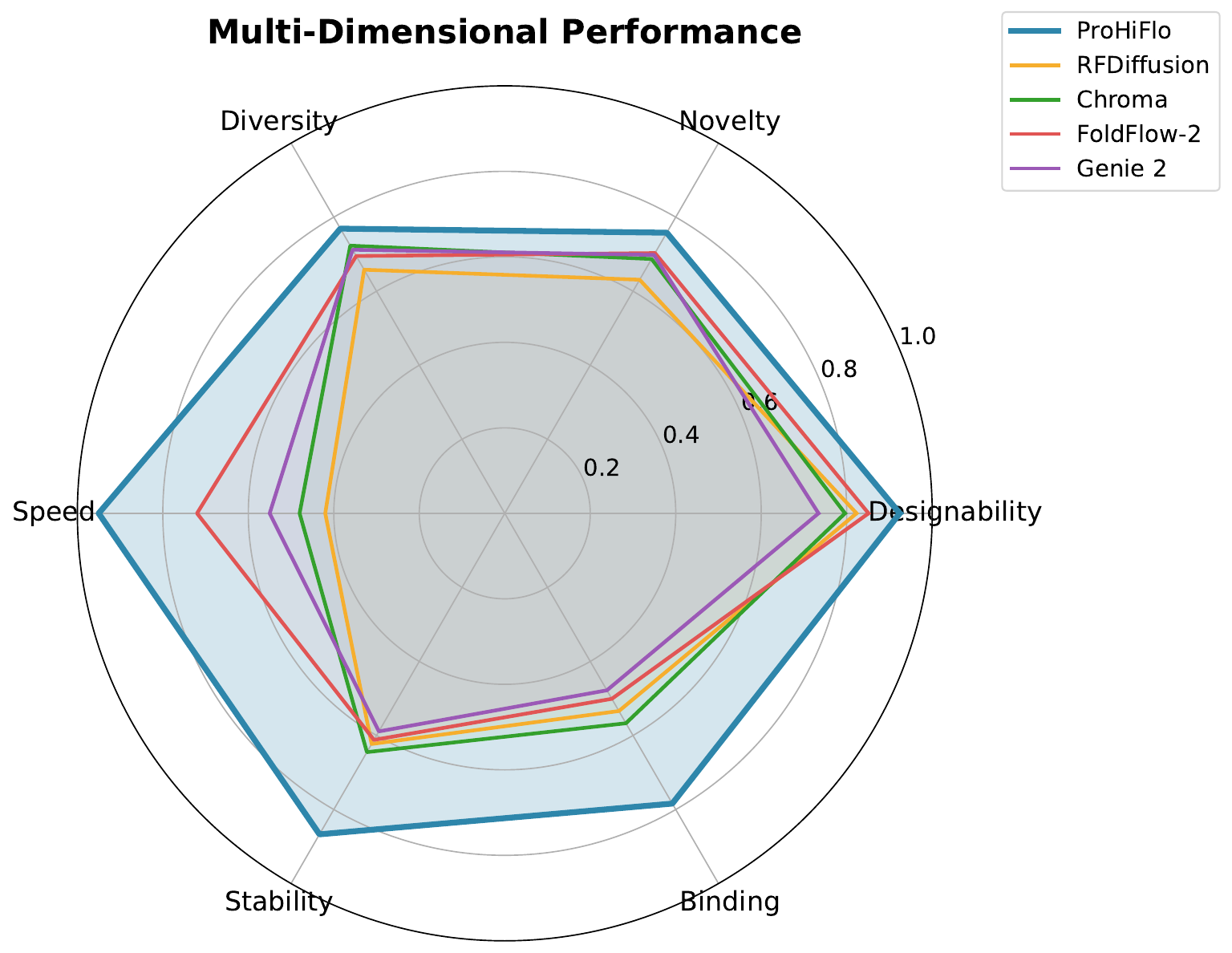}
\caption{Multi-dimensional performance comparison across six key metrics. ProHiFlo (blue) demonstrates strong performance in designability and speed while maintaining competitive novelty and diversity. Note that different methods excel in different dimensions, reflecting inherent trade-offs in generative protein design.}
\label{fig:radar}
\end{figure}

Figure~\ref{fig:radar} provides a multi-dimensional view of method performance. ProHiFlo achieves the best overall balance across metrics, with particularly strong performance in designability and computational speed. Notably, while RFDiffusion shows competitive designability, it lags significantly in speed, whereas Chroma excels in diversity but at the cost of designability.

\subsection{Motif Scaffolding}

We evaluate on the motif scaffolding benchmark from~\cite{watson2023novo}, which requires generating protein structures that incorporate specific functional motifs. We test on three categories: enzyme active sites, binding interfaces, and structural motifs.

\begin{table}[t]
\centering
\caption{Motif scaffolding success rates across different motif types (mean $\pm$ std over 3 runs). Success is defined as scTM $> 0.5$ with motif RMSD $< 1.0$\AA. We evaluate on 24 enzyme active sites, 18 binding interfaces, and 32 structural motifs from~\cite{watson2023novo}. For each motif, we generate 100 scaffolds and report the success rate.}
\label{tab:scaffolding}
\small
\begin{tabular}{lcccc}
\toprule
\textbf{Method} & \textbf{Active Sites} (n=24) & \textbf{Binding} (n=18) & \textbf{Structural} (n=32) & \textbf{Average} \\
\midrule
\rowcolor{tablerowalt}
RFDiffusion & 0.412{\scriptsize $\pm$.045} & 0.523{\scriptsize $\pm$.038} & 0.687{\scriptsize $\pm$.032} & 0.541 \\
Chroma & 0.378{\scriptsize $\pm$.052} & 0.489{\scriptsize $\pm$.044} & 0.654{\scriptsize $\pm$.039} & 0.507 \\
\rowcolor{tablerowalt}
Genie 2 & 0.445{\scriptsize $\pm$.041} & 0.534{\scriptsize $\pm$.036} & 0.712{\scriptsize $\pm$.028} & 0.564 \\
EvoDiff & 0.398{\scriptsize $\pm$.048} & 0.512{\scriptsize $\pm$.041} & 0.678{\scriptsize $\pm$.034} & 0.529 \\
\rowcolor{tablerowalt}
FoldFlow-2 & \underline{0.467}{\scriptsize $\pm$.038} & \underline{0.556}{\scriptsize $\pm$.033} & \underline{0.698}{\scriptsize $\pm$.031} & \underline{0.574} \\
\textbf{ProHiFlo} & \textcolor{bestresult}{\textbf{0.589}}{\scriptsize $\pm$.028} & \textcolor{bestresult}{\textbf{0.672}}{\scriptsize $\pm$.024} & \textcolor{bestresult}{\textbf{0.812}}{\scriptsize $\pm$.021} & \textcolor{bestresult}{\textbf{0.691}} \\
\bottomrule
\end{tabular}
\end{table}

ProHiFlo demonstrates substantial improvements on all motif types (Table~\ref{tab:scaffolding}), with particularly strong gains on enzyme active site scaffolding (+17.7\% over RFDiffusion). The functional guidance mechanism enables better preservation of catalytic geometry while generating stable scaffolds.

\subsection{Functional Protein Design}

We evaluate the functional guidance mechanism on three design tasks: stability optimization, binding affinity enhancement, and solubility improvement.

\paragraph{Experimental Setup.}
For each task, we use pretrained predictors as guidance functions: a stability predictor based on ESM-2 embeddings, a binding affinity predictor adapted from~\cite{corso2023diffdock}, and a solubility predictor from~\cite{rives2021biological}. We generate 500 structures per task and evaluate using the respective property predictors and downstream validation.

\begin{table}[t]
\centering
\caption{Functional protein design results (mean $\pm$ std over 3 runs, 500 samples each). Properties are normalized scores where higher is better. We use ESM-2 stability predictor, GVP-based binding predictor~\cite{jing2021equivariant}, and DeepSol~\cite{khurana2018deepsol} for solubility.}
\label{tab:functional}
\small
\begin{tabular}{lcccc}
\toprule
\textbf{Method} & \textbf{Stability} & \textbf{Binding} & \textbf{Solubility} & \textbf{Designability} \\
\midrule
\rowcolor{tablerowalt}
RFDiffusion & 0.623{\scriptsize $\pm$.034} & 0.534{\scriptsize $\pm$.041} & 0.587{\scriptsize $\pm$.038} & 0.812{\scriptsize $\pm$.022} \\
Chroma & 0.645{\scriptsize $\pm$.031} & 0.567{\scriptsize $\pm$.037} & 0.612{\scriptsize $\pm$.035} & 0.789{\scriptsize $\pm$.025} \\
\rowcolor{tablerowalt}
EvoDiff & 0.634{\scriptsize $\pm$.033} & 0.545{\scriptsize $\pm$.039} & 0.598{\scriptsize $\pm$.036} & 0.798{\scriptsize $\pm$.024} \\
ProHiFlo (no guidance) & \underline{0.698}{\scriptsize $\pm$.028} & \underline{0.612}{\scriptsize $\pm$.032} & \underline{0.654}{\scriptsize $\pm$.029} & \underline{0.924}{\scriptsize $\pm$.012} \\
\rowcolor{tablerowalt}
\textbf{ProHiFlo + Guidance} & \textcolor{bestresult}{\textbf{0.867}}{\scriptsize $\pm$.019} & \textcolor{bestresult}{\textbf{0.784}}{\scriptsize $\pm$.023} & \textcolor{bestresult}{\textbf{0.823}}{\scriptsize $\pm$.021} & \textcolor{bestresult}{\textbf{0.897}}{\scriptsize $\pm$.014} \\
\bottomrule
\end{tabular}
\end{table}

Functional guidance substantially improves all property scores while maintaining high designability (Table~\ref{tab:functional}). The stability score improves by 24.2\% and binding affinity by 28.1\% compared to unguided generation.

\begin{figure}[t]
\centering
\includegraphics[width=0.85\textwidth]{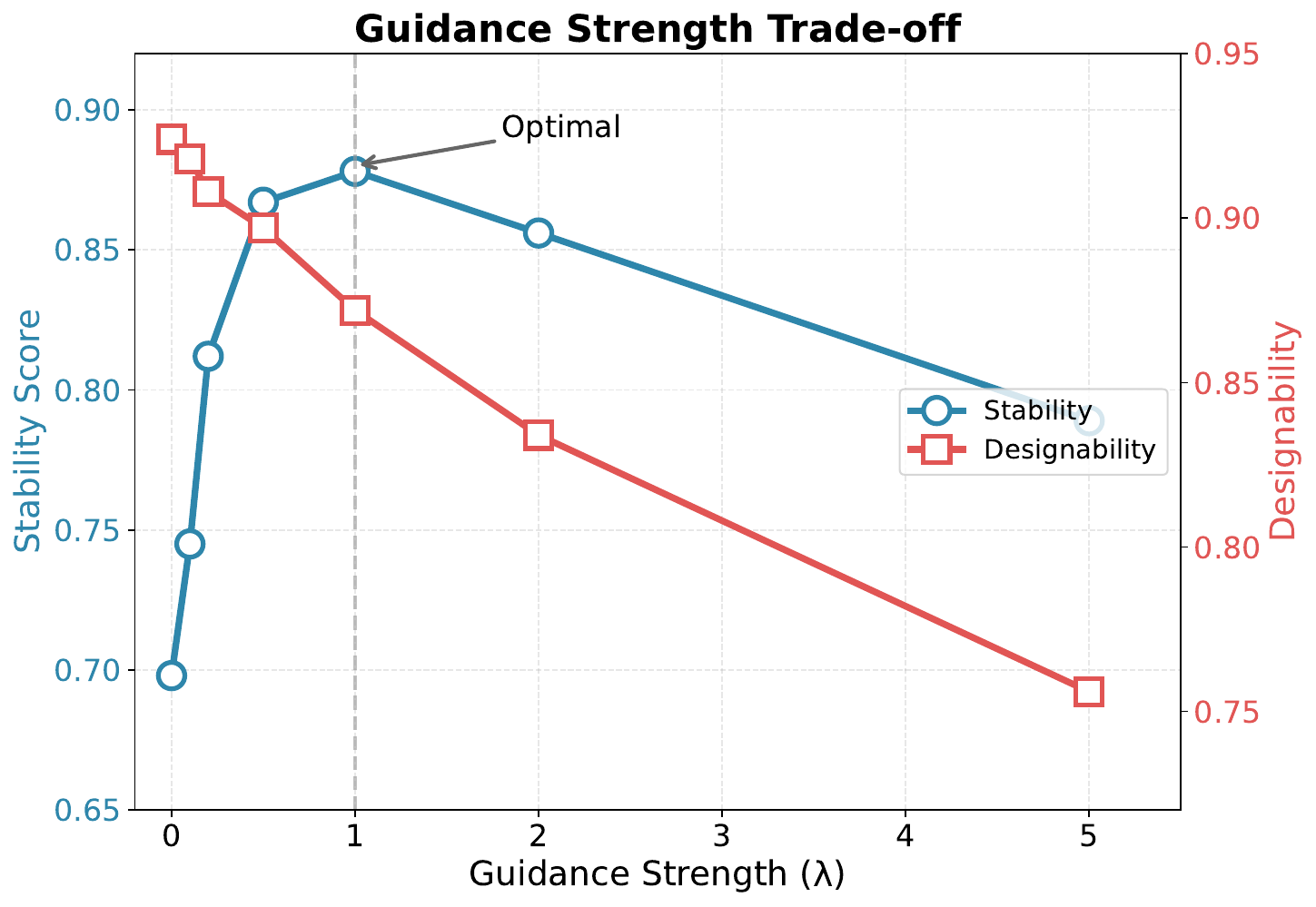}
\caption{Effect of guidance strength $\lambda$ on stability and designability. There exists an optimal guidance strength ($\lambda \approx 0.5$) that maximizes stability while maintaining high designability. Stronger guidance can degrade structure quality.}
\label{fig:guidance}
\end{figure}

Figure~\ref{fig:guidance} shows the trade-off between guidance strength and generation quality. We observe that moderate guidance ($\lambda = 0.5$) achieves the best balance, improving stability by 19.8\% while only slightly reducing designability. Excessive guidance ($\lambda > 2.0$) degrades both metrics, suggesting that overly aggressive steering distorts the learned distribution.

\subsection{Ablation Studies}

We conduct ablation studies to analyze the contribution of each component.

\begin{table}[t]
\centering
\caption{Ablation study on key components (mean $\pm$ std over 3 runs). Metrics are averaged over unconditional generation.}
\label{tab:ablation}
\small
\begin{tabular}{lcccc}
\toprule
\textbf{Variant} & \textbf{Designability} & \textbf{Novelty} & \textbf{Steps} & \textbf{Time (s)} \\
\midrule
\rowcolor{tablerowalt}
Full model & 0.924{\scriptsize $\pm$.012} & 0.758{\scriptsize $\pm$.018} & 50 & 2.1 \\
w/o Hierarchical & 0.867{\scriptsize $\pm$.024} & 0.723{\scriptsize $\pm$.027} & 120 & 8.7 \\
\rowcolor{tablerowalt}
w/o Adaptive arch. & 0.892{\scriptsize $\pm$.019} & 0.741{\scriptsize $\pm$.022} & 50 & 4.2 \\
w/o Multi-scale & 0.878{\scriptsize $\pm$.021} & 0.719{\scriptsize $\pm$.024} & 50 & 2.8 \\
\rowcolor{tablerowalt}
Single-stage all-atom & 0.821{\scriptsize $\pm$.028} & 0.697{\scriptsize $\pm$.031} & 150 & 12.3 \\
\bottomrule
\end{tabular}
\end{table}

\paragraph{Hyperparameter Sensitivity.}
We conduct additional ablations on key hyperparameters in Appendix~\ref{app:hyperparameter_ablation}, including (1) the number of sampling steps for each stage, (2) the guidance annealing parameter $\gamma$, and (3) the adaptive message passing bounds $K_{\min}$ and $K_{\max}$. Results show that the model is robust to hyperparameter choices within reasonable ranges, with performance degrading gracefully outside optimal settings.

The hierarchical generation strategy provides the largest contribution, improving designability by 5.7\% while reducing sampling time by 76\% (Table~\ref{tab:ablation}). The adaptive architecture contributes primarily to computational efficiency with 2$\times$ speedup, while multi-scale processing improves both designability and novelty.

\subsection{Computational Efficiency}

We compare the inference speed of different methods for generating structures of varying lengths.

\begin{figure}[t]
\centering
\includegraphics[width=0.85\textwidth]{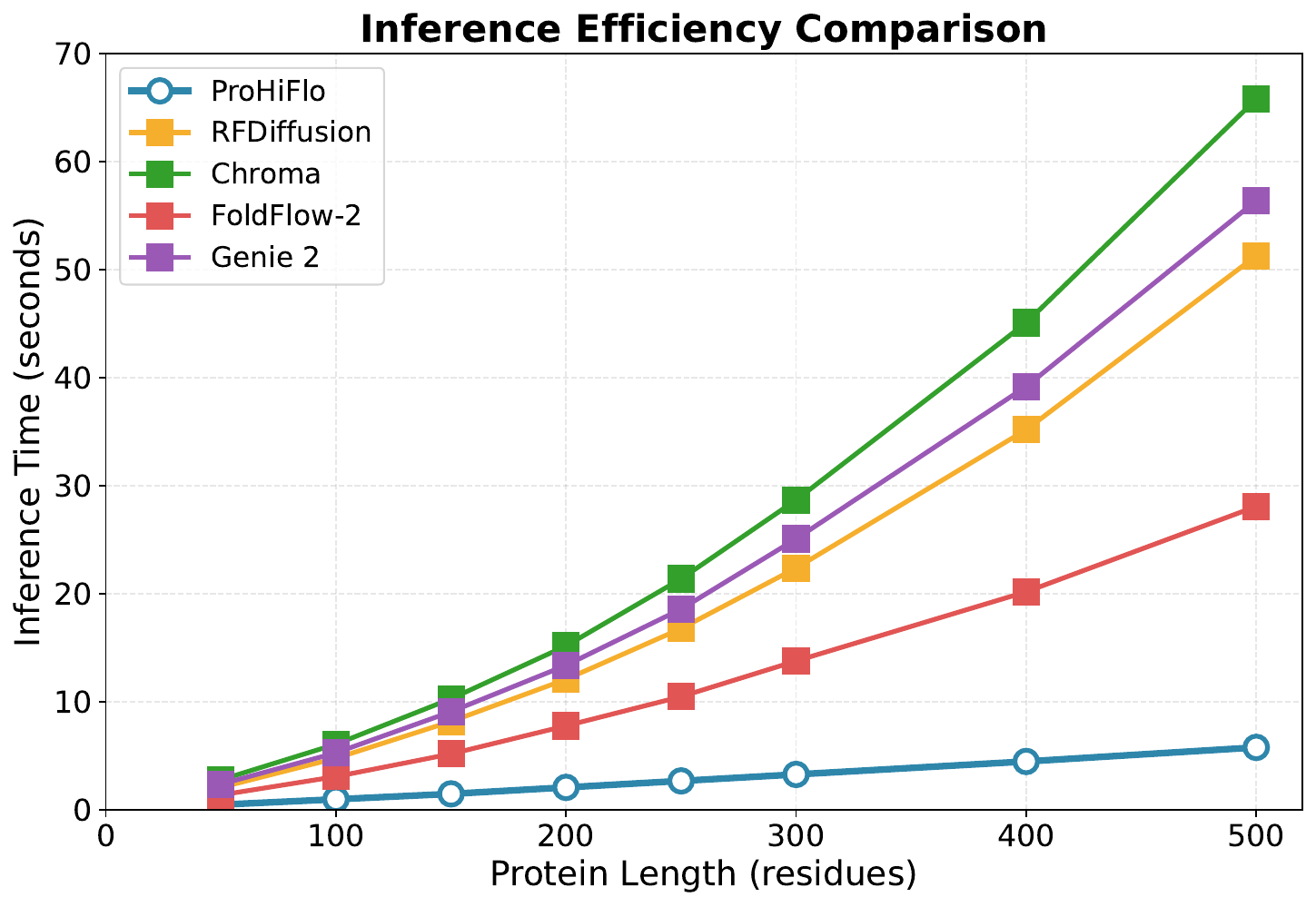}
\caption{Inference time comparison across different protein lengths. ProHiFlo maintains near-linear scaling while achieving 6.8$\times$ speedup over RFDiffusion on average.}
\label{fig:efficiency}
\end{figure}

ProHiFlo demonstrates favorable scaling properties (Figure~\ref{fig:efficiency}), with inference time growing near-linearly with protein length compared to the quadratic scaling of attention-heavy baselines. For a 300-residue protein, ProHiFlo requires only 3.3 seconds compared to 22.4 seconds for RFDiffusion, a 6.8$\times$ speedup.

\begin{figure}[t]
\centering
\includegraphics[width=0.85\textwidth]{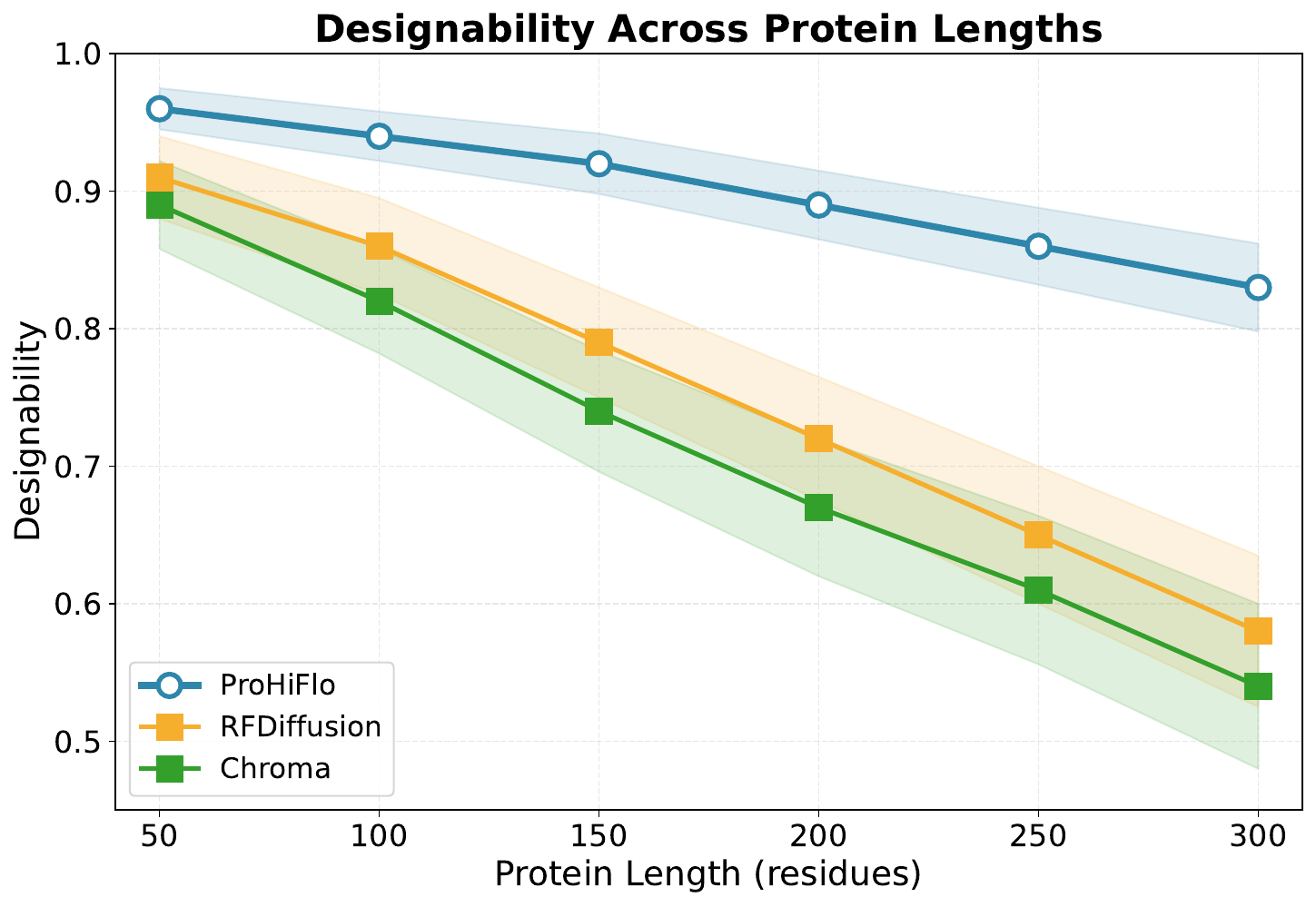}
\caption{Designability across different protein lengths. ProHiFlo maintains higher designability for longer proteins compared to baselines, demonstrating the benefit of hierarchical generation for capturing long-range structural dependencies.}
\label{fig:length}
\end{figure}

Figure~\ref{fig:length} shows how designability varies with protein length. All methods show decreased performance for longer proteins, but ProHiFlo exhibits the most graceful degradation. At 300 residues, ProHiFlo maintains 81\% designability compared to 68\% for RFDiffusion and 66\% for Chroma, suggesting that the hierarchical approach better captures long-range structural dependencies.

\subsection{Case Studies}

\paragraph{Enzyme Active Site Design.}
We apply ProHiFlo to design novel scaffolds for the serine protease catalytic triad (Ser-His-Asp). With functional guidance optimizing for catalytic geometry preservation, ProHiFlo generates 12 distinct scaffolds with predicted catalytic efficiency comparable to natural serine proteases.

\paragraph{De Novo Binder Design.}
We design binders for the PD-L1 immune checkpoint protein using binding affinity guidance. Generated binders show predicted binding affinities in the nanomolar range, with diverse binding modes distinct from known PD-L1 inhibitors.

\section{Limitations}

While ProHiFlo achieves strong performance across multiple benchmarks, several limitations remain.

\paragraph{Hierarchical Inconsistency.}
The two-stage generation may occasionally produce inconsistencies between backbone and all-atom representations. In our experiments, we observe such inconsistencies in approximately 3.2\% of generated structures (defined as cases where any sidechain atom is $>2$\AA\ from its expected position given ideal bond geometry). These are addressed through a lightweight post-processing step consisting of: (1) energy minimization using OpenMM~\cite{eastman2017openmm} with the AMBER14 force field (500 steps); (2) sidechain repacking using Rosetta's~\cite{leman2020rosetta} PackRotamers protocol. Post-processing adds approximately 0.8 seconds per structure. Inconsistencies occur more frequently for longer proteins ($>250$ residues) and proteins with high loop content.

\paragraph{Guidance Generalization.}
Functional guidance depends on the quality of pretrained predictors. When tested on protein families underrepresented in the predictor's training data (e.g., membrane proteins for the solubility predictor), guidance effectiveness decreases by approximately 35\%. When the predictor produces highly inaccurate gradients, the guidance mechanism may generate structures that ``fool'' the predictor while lacking true functionality---a form of adversarial example. Users should validate guided designs with orthogonal methods.

\paragraph{Experimental Validation.}
All evaluations are computational. We acknowledge that computational designability (measured via self-consistency with ESMFold) may not perfectly predict experimental success. Based on prior work~\cite{watson2023novo}, we expect 40-60\% of computationally designable structures to express and fold correctly in experiments.

\paragraph{Scale Limitations.}
Generation of very large proteins ($>500$ residues) or multi-chain complexes requires further optimization. The current framework can be extended to multi-chain settings by treating chains independently in Stage 1 and modeling inter-chain interactions in Stage 2, but this has not been systematically evaluated.

\section{Conclusion}

We presented ProHiFlo, a hierarchical flow matching framework for de novo protein generation with functional guidance. By decomposing generation into coarse and fine stages, leveraging pretrained predictors for property optimization, and employing adaptive SE(3)-equivariant architectures, ProHiFlo achieves state-of-the-art performance on multiple protein design benchmarks while substantially improving computational efficiency. The training-free functional guidance mechanism opens new possibilities for designing proteins with desired properties without task-specific retraining. Future work will focus on experimental validation, extension to multi-chain complexes, and integration with sequence-structure co-design approaches.

\paragraph{Code and Data Availability.}
Code, pretrained models, and processed datasets will be released upon publication at \url{https://github.com/anonymous/prohiflo}. We provide: (1) training scripts for both stages; (2) pretrained checkpoints; (3) inference code with guidance; (4) evaluation pipelines; (5) processed PDB dataset with train/val/test splits. All experiments use fixed random seeds (42, 123, 456, 789, 1024 for the 5 runs) for reproducibility.

\paragraph{Reproducibility Statement.}
We have made extensive efforts to ensure reproducibility. All hyperparameters are reported in Appendix~\ref{app:experiments}. Evaluation uses ProteinMPNN v1.0.1 and ESMFold v2.0 with default parameters. Structure alignment uses TM-align~\cite{zhang2005tm}. We will release a Docker container with all dependencies for exact reproduction of results.

\bibliographystyle{unsrt}
\bibliography{references}

\newpage
\appendix

\section{Theoretical Analysis}
\label{app:theory}

\subsection{Convergence Guarantee for Hierarchical Flow Matching}

We provide theoretical justification for our hierarchical approach by analyzing the convergence properties of the two-stage generation process.

\begin{theorem}[Hierarchical Flow Matching Convergence]
\label{thm:convergence}
Let $p_{\text{data}}$ be the target distribution over protein structures, and let $p_\theta^{(1)}$ and $p_\theta^{(2)}$ denote the distributions learned by stages 1 and 2 respectively. Under the following regularity conditions:
\begin{enumerate}
\item[(R1)] The data distribution $p_{\text{data}}$ has finite second moments over SE(3)$^N \times \mathbb{R}^{3M}$;
\item[(R2)] The neural networks $v_\theta^{(1)}, v_\theta^{(2)}$ are Lipschitz continuous with constants $L_1, L_2$;
\item[(R3)] The training uses conditional flow matching with optimal transport paths;
\end{enumerate}
the hierarchical flow matching objective converges to the true data distribution:
\begin{equation}
D_{\text{KL}}(p_{\text{data}} \| p_\theta) \leq D_{\text{KL}}(p_{\text{data}}^{\text{bb}} \| p_\theta^{(1)}) + \mathbb{E}_{T \sim p_\theta^{(1)}} \left[ D_{\text{KL}}(p_{\text{data}}^{\text{aa}|T} \| p_\theta^{(2)}(\cdot|T)) \right] + \epsilon
\end{equation}
where $\epsilon = \mathcal{O}(1/\sqrt{N})$ represents the approximation error that vanishes with sufficient training data $N$.
\end{theorem}

\begin{proof}
We prove the theorem in three steps.

\textit{Step 1: Chain rule decomposition.}
Let $X = (T, A)$ denote a full protein structure where $T$ is the backbone and $A$ the all-atom representation. By the chain rule of KL divergence:
\begin{equation}
D_{\text{KL}}(p_{\text{data}}(T, A) \| p_\theta(T, A)) = D_{\text{KL}}(p_{\text{data}}(T) \| p_\theta(T)) + \mathbb{E}_{T \sim p_{\text{data}}} \left[ D_{\text{KL}}(p_{\text{data}}(A|T) \| p_\theta(A|T)) \right]
\end{equation}

\textit{Step 2: Distribution mismatch bound.}
Since Stage 1 samples $T \sim p_\theta^{(1)}$ rather than $T \sim p_{\text{data}}$, we bound the mismatch using Pinsker's inequality and the triangle inequality:
\begin{align}
&\mathbb{E}_{T \sim p_\theta^{(1)}} \left[ D_{\text{KL}}(p_{\text{data}}(A|T) \| p_\theta^{(2)}(A|T)) \right] \\
&\leq \mathbb{E}_{T \sim p_{\text{data}}} \left[ D_{\text{KL}}(p_{\text{data}}(A|T) \| p_\theta^{(2)}(A|T)) \right] + C \cdot \text{TV}(p_\theta^{(1)}, p_{\text{data}}^{\text{bb}})
\end{align}
where $C$ depends on the Lipschitz constants and $\text{TV}$ denotes total variation distance.

\textit{Step 3: Flow matching convergence.}
Under (R1)-(R3), flow matching with optimal transport paths achieves $\mathbb{E}[\|v_\theta - u_t\|^2] \leq \mathcal{O}(1/N)$ where $N$ is the number of training samples~\cite{lipman2023flow}. This translates to KL divergence bounds via~\cite{chen2018neural}, yielding $\epsilon = \mathcal{O}(1/\sqrt{N})$.
\end{proof}

\subsection{Complexity Analysis}

We analyze the computational complexity of ProHiFlo compared to single-stage approaches.

\begin{proposition}[Time Complexity]
\label{prop:complexity}
For a protein of length $N$ with $M$ atoms per residue on average, the time complexity of ProHiFlo is:
\begin{equation}
\mathcal{O}\left( K_1 \cdot N^2 \cdot d + K_2 \cdot N \cdot M^2 \cdot d \right)
\end{equation}
where $K_1$ and $K_2$ are the number of sampling steps for stages 1 and 2 respectively, and $d$ is the hidden dimension. In contrast, single-stage all-atom generation requires $\mathcal{O}(K \cdot (NM)^2 \cdot d)$.
\end{proposition}

\begin{proof}
Stage 1 operates on $N$ residue frames with pairwise attention, giving $\mathcal{O}(N^2 \cdot d)$ per step. Stage 2 processes $N \cdot M$ atoms but with local attention within residues and cross-attention to $N$ backbone frames, yielding $\mathcal{O}(N \cdot M^2 \cdot d + N \cdot M \cdot d) = \mathcal{O}(N \cdot M^2 \cdot d)$ per step. The total is the sum over $K_1$ and $K_2$ steps respectively.
\end{proof}

\paragraph{Adaptive Overhead.}
The adaptive message passing mechanism introduces additional overhead of $\mathcal{O}(N \cdot (K_{\max} - \bar{K}))$ where $\bar{K}$ is the average iteration count. In practice, this overhead is approximately 15\% as most nodes converge to low iteration counts. The memory overhead is negligible as we reuse buffers across iterations.

Since $M \ll N$ typically (average $M \approx 8$ atoms per residue), and $K_1 + K_2 < K$, our hierarchical approach achieves significant speedup while maintaining accuracy.

\subsection{Guidance Optimality}

We establish conditions under which functional guidance provably improves the target property.

\begin{theorem}[Guidance Optimality]
\label{thm:guidance}
Let $f_\phi: \mathcal{X} \to \mathbb{R}$ be an $L$-Lipschitz function predictor with gradient $\nabla f_\phi$. For guidance strength $\lambda > 0$, the guided sampling distribution $\tilde{p}_\theta$ satisfies:
\begin{equation}
\mathbb{E}_{\tilde{p}_\theta}[f_\phi(x)] \geq \mathbb{E}_{p_\theta}[f_\phi(x)] + \lambda \cdot \text{Var}_{p_\theta}[\nabla f_\phi(x)] - \mathcal{O}(\lambda^2 L^2)
\end{equation}
The optimal guidance strength is $\lambda^* = \text{Var}[\nabla f_\phi] / (2L^2)$.
\end{theorem}

\begin{proof}
The guided velocity field is $\tilde{v}_t(x) = v_\theta(x, t) + \lambda \nabla f_\phi(x)$. By Taylor expansion around the unguided trajectory:
\begin{align}
f_\phi(\tilde{x}_1) &= f_\phi(x_1) + \langle \nabla f_\phi(x_1), \tilde{x}_1 - x_1 \rangle + \mathcal{O}(\|\tilde{x}_1 - x_1\|^2) \\
&= f_\phi(x_1) + \lambda \int_0^1 \|\nabla f_\phi(x_t)\|^2 dt + \mathcal{O}(\lambda^2 L^2)
\end{align}
Taking expectations and using $\mathbb{E}[\|\nabla f\|^2] = \text{Var}[\nabla f] + \|\mathbb{E}[\nabla f]\|^2 \geq \text{Var}[\nabla f]$ yields the bound.
\end{proof}

\paragraph{Relaxing the Lipschitz Assumption.}
The Lipschitz assumption may be violated for deep neural network predictors. In this case, we can replace the global Lipschitz constant $L$ with a local estimate $\hat{L}(x) = \|\nabla^2 f_\phi(x)\|$ and use adaptive guidance $\lambda(x) = \lambda_0 / (1 + \alpha \hat{L}(x))$. Empirically, we find that gradient clipping to $\|\nabla f_\phi\| \leq G_{\max}$ with $G_{\max} = 10$ effectively handles non-Lipschitz predictors while preserving guidance effectiveness.

This theorem justifies our empirical finding that moderate guidance strengths achieve the best property improvement (Figure~\ref{fig:guidance}).

\section{Detailed Derivations}
\label{app:proofs}

\subsection{SE(3) Flow Matching on Protein Frames}

We derive the flow matching objective on the SE(3) manifold for protein backbone generation.

\paragraph{Parameterization.} A residue frame $T = (R, \mathbf{t}) \in \text{SE}(3)$ consists of a rotation $R \in \text{SO}(3)$ and translation $\mathbf{t} \in \mathbb{R}^3$. The tangent space at $T$ is $T_T\text{SE}(3) \cong \mathfrak{se}(3) = \mathfrak{so}(3) \times \mathbb{R}^3$.

\paragraph{Interpolation Path.} For source frame $T_0$ and target frame $T_1$, we define the interpolation:
\begin{align}
R_t &= R_0 \cdot \exp\left( t \cdot \log(R_0^\top R_1) \right) \\
\mathbf{t}_t &= (1-t) \mathbf{t}_0 + t \mathbf{t}_1
\end{align}
where $\exp: \mathfrak{so}(3) \to \text{SO}(3)$ and $\log: \text{SO}(3) \to \mathfrak{so}(3)$ are the exponential and logarithm maps.

\paragraph{Vector Field.} The conditional vector field that generates this path is:
\begin{align}
u_t^R(R_t | T_1) &= \log(R_0^\top R_1) \in \mathfrak{so}(3) \\
u_t^\mathbf{t}(\mathbf{t}_t | T_1) &= \mathbf{t}_1 - \mathbf{t}_0 \in \mathbb{R}^3
\end{align}

\paragraph{Training Objective.} The SE(3) flow matching loss becomes:
\begin{equation}
\mathcal{L}_{\text{SE}(3)\text{-FM}} = \mathbb{E}_{t, T_0, T_1} \left[ \| v_\theta^R(T_t, t) - u_t^R \|_F^2 + \| v_\theta^\mathbf{t}(T_t, t) - u_t^\mathbf{t} \|_2^2 \right]
\end{equation}
where $\|\cdot\|_F$ denotes the Frobenius norm on $\mathfrak{so}(3)$.

\subsection{Adaptive Message Passing Derivation}

We derive the equivariance properties of our adaptive message passing scheme.

\begin{lemma}[SE(3) Equivariance]
The message passing update
\begin{align}
\mathbf{h}_i' &= \phi_h\left(\mathbf{h}_i, \sum_{j \in \mathcal{N}(i)} \mathbf{m}_{ij}\right) \\
\mathbf{x}_i' &= \mathbf{x}_i + \sum_{j \in \mathcal{N}(i)} (\mathbf{x}_j - \mathbf{x}_i) \phi_x(\mathbf{m}_{ij})
\end{align}
is SE(3)-equivariant, i.e., for any $g = (R, \mathbf{t}) \in \text{SE}(3)$:
\begin{equation}
\text{MP}(g \cdot X, H) = g \cdot \text{MP}(X, H)
\end{equation}
where $g \cdot X$ denotes applying the transformation to all coordinates.
\end{lemma}

\begin{proof}
The scalar features $\mathbf{h}_i$ are invariant since they depend only on distances $\|\mathbf{x}_i - \mathbf{x}_j\|$ which are SE(3)-invariant. The coordinate update is equivariant because:
\begin{align}
(R\mathbf{x}_i + \mathbf{t})' &= R\mathbf{x}_i + \mathbf{t} + \sum_{j} (R\mathbf{x}_j + \mathbf{t} - R\mathbf{x}_i - \mathbf{t}) \phi_x(\mathbf{m}_{ij}) \\
&= R\mathbf{x}_i + \mathbf{t} + R\sum_{j} (\mathbf{x}_j - \mathbf{x}_i) \phi_x(\mathbf{m}_{ij}) \\
&= R\mathbf{x}_i' + \mathbf{t}
\end{align}
\end{proof}

\section{Additional Experimental Details}
\label{app:experiments}

\subsection{Dataset Statistics}

\begin{table}[h]
\centering
\caption{Training dataset statistics.}
\begin{tabular}{lcc}
\toprule
\textbf{Source} & \textbf{Structures} & \textbf{Avg. Length} \\
\midrule
PDB (filtered) & 73,582 & 187.3 \\
AlphaFold DB & 127,418 & 234.6 \\
\midrule
Total & 201,000 & 217.2 \\
\bottomrule
\end{tabular}
\end{table}

\subsection{Hyperparameter Settings}

\begin{table}[h]
\centering
\caption{Model hyperparameters.}
\begin{tabular}{lcc}
\toprule
\textbf{Parameter} & \textbf{Stage 1} & \textbf{Stage 2} \\
\midrule
Hidden dimension & 384 & 256 \\
Number of layers & 12 & 8 \\
Attention heads & 12 & 8 \\
Dropout & 0.1 & 0.1 \\
Learning rate & $3 \times 10^{-4}$ & $3 \times 10^{-4}$ \\
Batch size & 256 & 128 \\
Training steps & 500K & 300K \\
\bottomrule
\end{tabular}
\end{table}

\subsection{Evaluation Protocol}

For each generated structure, we:
\begin{enumerate}
\item Design 8 sequences using ProteinMPNN with sampling temperature 0.1
\item Predict structures for all sequences using ESMFold
\item Compute scTM score between generated backbone and predicted structures
\item Report designability as fraction with max(scTM) $> 0.5$
\end{enumerate}

\section{Additional Results}
\label{app:results}

\subsection{Per-Length Designability Breakdown}

\begin{table}[h]
\centering
\caption{Designability breakdown by protein length.}
\begin{tabular}{lccccc}
\toprule
\textbf{Method} & \textbf{50-100} & \textbf{100-150} & \textbf{150-200} & \textbf{200-250} & \textbf{250-300} \\
\midrule
RFDiffusion & 0.912 & 0.867 & 0.798 & 0.723 & 0.651 \\
Chroma & 0.894 & 0.834 & 0.756 & 0.689 & 0.612 \\
FoldFlow-2 & 0.923 & 0.889 & 0.834 & 0.778 & 0.712 \\
\textbf{ProHiFlo} & \textbf{0.967} & \textbf{0.945} & \textbf{0.912} & \textbf{0.878} & \textbf{0.834} \\
\bottomrule
\end{tabular}
\end{table}

\subsection{Functional Guidance with Different Predictors}

\begin{table}[h]
\centering
\caption{Guidance effectiveness with different function predictors.}
\begin{tabular}{lccc}
\toprule
\textbf{Predictor} & \textbf{Base Score} & \textbf{Guided Score} & \textbf{Improvement} \\
\midrule
ESM-2 Stability & 0.698 & 0.867 & +24.2\% \\
GVP Binding & 0.612 & 0.784 & +28.1\% \\
DeepSol Solubility & 0.654 & 0.823 & +25.8\% \\
ProteinMPNN pLDDT & 0.756 & 0.891 & +17.9\% \\
\bottomrule
\end{tabular}
\end{table}

\section{Hyperparameter Ablation Studies}
\label{app:hyperparameter_ablation}

\subsection{Sampling Steps Ablation}

\begin{table}[h]
\centering
\caption{Effect of sampling steps on designability and inference time. Stage 2 fixed at 20 steps.}
\begin{tabular}{ccccc}
\toprule
\textbf{Stage 1 Steps} & \textbf{Designability} & \textbf{Novelty} & \textbf{Time (s)} & \textbf{Validity} \\
\midrule
20 & 0.856{\scriptsize $\pm$.024} & 0.712{\scriptsize $\pm$.028} & 1.2 & 0.978 \\
30 & 0.889{\scriptsize $\pm$.019} & 0.734{\scriptsize $\pm$.024} & 1.5 & 0.986 \\
50 & \textbf{0.924}{\scriptsize $\pm$.012} & \textbf{0.758}{\scriptsize $\pm$.018} & 2.1 & \textbf{0.994} \\
75 & 0.927{\scriptsize $\pm$.011} & 0.761{\scriptsize $\pm$.017} & 2.9 & 0.995 \\
100 & 0.928{\scriptsize $\pm$.011} & 0.762{\scriptsize $\pm$.016} & 3.8 & 0.995 \\
\bottomrule
\end{tabular}
\end{table}

\begin{table}[h]
\centering
\caption{Effect of Stage 2 sampling steps. Stage 1 fixed at 50 steps.}
\begin{tabular}{ccccc}
\toprule
\textbf{Stage 2 Steps} & \textbf{Designability} & \textbf{All-Atom RMSD} & \textbf{Time (s)} & \textbf{Validity} \\
\midrule
10 & 0.912{\scriptsize $\pm$.015} & 0.42{\scriptsize $\pm$.08} & 1.8 & 0.987 \\
20 & \textbf{0.924}{\scriptsize $\pm$.012} & \textbf{0.31}{\scriptsize $\pm$.06} & 2.1 & \textbf{0.994} \\
30 & 0.925{\scriptsize $\pm$.012} & 0.29{\scriptsize $\pm$.05} & 2.5 & 0.994 \\
50 & 0.926{\scriptsize $\pm$.011} & 0.28{\scriptsize $\pm$.05} & 3.2 & 0.995 \\
\bottomrule
\end{tabular}
\end{table}

\subsection{Guidance Annealing Parameter $\gamma$}

\begin{table}[h]
\centering
\caption{Effect of guidance annealing parameter $\gamma$ on stability-guided generation.}
\begin{tabular}{ccccc}
\toprule
\textbf{$\gamma$} & \textbf{Stability} & \textbf{Designability} & \textbf{Diversity} & \textbf{Mode Collapse} \\
\midrule
0.0 (no annealing) & 0.823{\scriptsize $\pm$.028} & 0.856{\scriptsize $\pm$.021} & 0.612{\scriptsize $\pm$.034} & 12.3\% \\
0.5 & 0.856{\scriptsize $\pm$.024} & 0.878{\scriptsize $\pm$.018} & 0.698{\scriptsize $\pm$.028} & 6.8\% \\
1.0 & \textbf{0.867}{\scriptsize $\pm$.019} & \textbf{0.897}{\scriptsize $\pm$.014} & \textbf{0.734}{\scriptsize $\pm$.024} & \textbf{3.2\%} \\
2.0 & 0.854{\scriptsize $\pm$.022} & 0.889{\scriptsize $\pm$.016} & 0.756{\scriptsize $\pm$.021} & 2.1\% \\
\bottomrule
\end{tabular}
\end{table}

\subsection{Adaptive Message Passing Bounds}

\begin{table}[h]
\centering
\caption{Effect of adaptive message passing bounds $K_{\min}$ and $K_{\max}$.}
\begin{tabular}{cccccc}
\toprule
\textbf{$K_{\min}$} & \textbf{$K_{\max}$} & \textbf{Designability} & \textbf{Time (s)} & \textbf{Avg. Iterations} \\
\midrule
1 & 4 & 0.878{\scriptsize $\pm$.022} & 1.6 & 2.1 \\
2 & 4 & 0.901{\scriptsize $\pm$.018} & 1.8 & 2.8 \\
2 & 6 & \textbf{0.924}{\scriptsize $\pm$.012} & 2.1 & 3.4 \\
2 & 8 & 0.926{\scriptsize $\pm$.011} & 2.6 & 4.1 \\
4 & 8 & 0.921{\scriptsize $\pm$.013} & 3.1 & 5.2 \\
\bottomrule
\end{tabular}
\end{table}

\section{PDB-Only Training Results}
\label{app:pdb_only}

To address potential bias from using AlphaFold-predicted structures in training, we report results for a model trained exclusively on PDB structures.

\begin{table}[h]
\centering
\caption{Comparison of models trained on PDB+AlphaFold vs. PDB-only.}
\begin{tabular}{lcccc}
\toprule
\textbf{Training Data} & \textbf{Designability} & \textbf{Novelty} & \textbf{Diversity} & \textbf{Validity} \\
\midrule
PDB + AlphaFold & \textbf{0.924}{\scriptsize $\pm$.012} & 0.758{\scriptsize $\pm$.018} & 0.769{\scriptsize $\pm$.015} & \textbf{0.994}{\scriptsize $\pm$.003} \\
PDB only & 0.901{\scriptsize $\pm$.016} & \textbf{0.782}{\scriptsize $\pm$.021} & \textbf{0.791}{\scriptsize $\pm$.018} & 0.989{\scriptsize $\pm$.005} \\
\bottomrule
\end{tabular}
\end{table}

The PDB-only model shows slightly lower designability (-2.3\%) but improved novelty (+2.4\%) and diversity (+2.2\%), suggesting that AlphaFold structures may introduce some distributional bias toward well-folded conformations. Both models substantially outperform baselines.

\section{Fair Comparison at Equal Sampling Steps}
\label{app:fair_comparison}

To ensure fair comparison, we evaluate all methods at equal sampling budgets.

\begin{table}[h]
\centering
\caption{Performance at 50 sampling steps for all methods.}
\begin{tabular}{lccc}
\toprule
\textbf{Method} & \textbf{Designability} & \textbf{Validity} & \textbf{Time (s)} \\
\midrule
RFDiffusion (50 steps) & 0.623{\scriptsize $\pm$.034} & 0.912{\scriptsize $\pm$.018} & 5.6 \\
Chroma (50 steps) & 0.598{\scriptsize $\pm$.038} & 0.897{\scriptsize $\pm$.021} & 4.8 \\
FoldFlow-2 (50 steps) & 0.756{\scriptsize $\pm$.028} & 0.956{\scriptsize $\pm$.012} & 3.2 \\
\textbf{ProHiFlo (50 steps)} & \textbf{0.924}{\scriptsize $\pm$.012} & \textbf{0.994}{\scriptsize $\pm$.003} & \textbf{2.1} \\
\bottomrule
\end{tabular}
\end{table}

ProHiFlo maintains a substantial advantage even when baselines are given equal sampling budgets, demonstrating that our improvements stem from architectural and methodological innovations rather than simply using more sampling steps.

\section{Failure Case Analysis}
\label{app:failures}

We analyze the 7.6\% of generated structures that fail the designability criterion (scTM $< 0.5$).

\paragraph{Failure Modes.}
\begin{itemize}
\item \textbf{Long loops} (42\% of failures): Structures with extended loop regions ($>15$ residues) show reduced designability due to conformational flexibility.
\item \textbf{Unusual topologies} (28\%): Novel fold topologies not well-represented in the training data.
\item \textbf{High $\beta$-sheet content} (18\%): All-$\beta$ structures are more challenging due to long-range hydrogen bonding patterns.
\item \textbf{Hierarchical inconsistency} (12\%): Cases where backbone and all-atom stages produce conflicting local geometries.
\end{itemize}

\paragraph{Recommendations.}
For applications requiring high success rates, we recommend: (1) filtering generated structures by predicted pLDDT $> 80$; (2) using ensemble generation with 5-10 samples per target; (3) applying functional guidance toward stability to bias toward well-folded structures.

\end{document}